\documentclass[11pt]{article}

\pdfoutput=1

\usepackage[numbers]{natbib}

\usepackage[T1]{fontenc}
\usepackage[latin9]{inputenc}
\usepackage{float}
\usepackage{fullpage,times,color}
\usepackage{boxedminipage}

\usepackage{amsmath,amsthm,amssymb}
\usepackage{mathtools}

\usepackage{graphicx}
\usepackage{ltxtable} 
\newcolumntype{L}{>{\centering\arraybackslash} m{0.04\columnwidth}} 
\newcolumntype{R}{>{\centering\arraybackslash} m{0.48\columnwidth}} 
\newcolumntype{S}{>{\centering\arraybackslash} m{0.32\columnwidth}} 

\usepackage{multirow}
\usepackage{array}

\usepackage{tikz}
\usepackage{pgfplots}

\newcommand{\cH}{{\mathcal H}}

\newcommand{\sign}{{\mathrm {sign}}}

\newtheorem{lemma}{Lemma}

\newtheorem{theorem}{Theorem}
\newtheorem{corollary}{Corollary}

\DeclareMathOperator*{\argmin}{argmin} 
 
\newcommand{\reals}{\mathbb{R}}

\newenvironment{myalgo}[1]%
{
\begin{center}
\begin{boxedminipage}{0.8\linewidth}
\begin{center}
\textbf{\texttt{#1}}
\end{center}
\rm
\begin{tabbing}
....\=...\=...\=...\=...\=  \+ \kill
} %
{\end{tabbing} 
\end{boxedminipage} \end{center} 
}

\title{SelfieBoost: A Boosting Algorithm for Deep Learning}
\author{Shai Shalev-Shwartz\thanks{School of Computer Science and
    Engineering, The Hebrew University, Jerusalem, Israel. 
This research is supported by Intel (ICRI-CI).
}  }
\date{}

\begin{document}

\maketitle

\begin{abstract}
  We describe and analyze a new boosting algorithm for deep learning
  called SelfieBoost. Unlike other boosting algorithms, like AdaBoost,
  which construct ensembles of classifiers, SelfieBoost boosts the
  accuracy of a single network. We prove a $\log(1/\epsilon)$
  convergence rate for SelfieBoost under some ``SGD success''
  assumption which seems to hold in practice.
\end{abstract}

\section{Introduction}

Deep learning, which involves training artificial neural networks with
many layers, becomes one of the most significant recent developments
in machine learning. Deep learning have shown very impressive
practical performance on a variety of domains
(e.g. \cite{lecun1995convolutional,HiOsTe06,ranzato2007unsupervised,bengio2007scaling,ColWes08,lee2009convolutional,LRMDCCDN12,krizhevsky2012imagenet,zeiler2013visualizing,dahl2013dropout,bengio2013representation}).

One of the most successful approaches for training neural networks is
Stochastic Gradient Descent (SGD) and its variants (see for example
\cite{lecun1998gradient,bousquet2008tradeoffs,Beng09,lecun2012efficient,sutskever2013importance,sutskever2013training,SSSSBD14}).
The two main advantages of SGD are the constant cost of each iteration
(which does not depend on the number of examples) and the ability to
overcome local minima. However, a major disadvantage of SGD is its
slow convergence. 

There have been several attempts to speed up the convergence rate of
SGD, such as momentum \cite{sutskever2013importance}, second order
information \cite{martens2010deep,duchi2011adaptive}, and variance
reducing methods
\cite{shalev2013stochastic,schmidt2013minimizing,johnson2013accelerating}. 

Another natural approach is to use SGD as a weak learner and apply a
boosting algorithm such as AdaBoost \cite{FreundSc95}. See for example
\cite{schwenk2000boosting}. The celebrated analysis of AdaBoost
guarantees that if at each boosting iteration the weak learner manages
to produce a classifier which is slightly better than a random guess
than after $O(\log(1/\epsilon))$ iterations, AdaBoost returns an
ensemble of classifiers whose training error is at most $\epsilon$. 

However, a major disadvantage of AdaBoost is that it outputs an
ensemble of classifiers. In the context of deep learning, this means
that at prediction time, we need to apply several neural networks on
every example to make the prediction. This leads to a rather slow
predictor. 

In this paper we present the SelfieBoost algorithm, which boosts the
performance of a single network. SelfieBoost can be applied with SGD
as its weak learner, and we prove that its convergence rate is 
$O(\log(1/\epsilon))$ provided that SGD manages to find a constant
accuracy solution at each boosting iteration. 

In a follow-up paper \cite{shalev2016minimizing}, we describe a different boosting
algorithm for deep learning. 

\section{The SelfieBoost Algorithm}

We now describe the SelfieBoost algorithm. We focus on a binary
classification problem in the realizable case. Let $(x_1,y_1),\ldots,(x_m,y_m)$ be the
training set, with $x_i \in \reals^d$ and $y_i \in \{\pm 1\}$. 
Let $\cH$ be the class of all functions that can be implemented by a
neural network of a certain architecture. 
%
%
We assume that there exists a network $f^* \in
\cH$ such that $y_i f^*(x_i) \ge 1$ for all $i \in [m]$.
Similarly to the AdaBoost algorithm, we maintain weights over the
examples that are proportional to minus the signed margin, $y_i
f_t(x_i)$. This focuses the algorithm on the hard cases. Focusing the learner on the mistakes of $f_t$ might cause
$f_{t+1}$ to work well on these examples but deteriorate on the
examples on which $f_t$ performs well. AdaBoost overcomes this problem
by remembering all the intermediate classifiers and predicting a
weighted majority of their predictions. In contrast, SelfieBoost
forgets the intermediate classifiers and outputs just the last
classifier. We therefore need another method to make sure that the
performance does not deteriorate. SelfieBoost achieves this goal by
regularizing $f_{t+1}$ so that its predictions will not be too far
from the predictions of $f_t$. 

\begin{myalgo}{SelfieBoost }
\textbf{Parameters:} Edge parameter, $\rho \in (0,\tfrac{1}{4})$, and number of iterations
$T$ \\
\textbf{Initialization:} Start with an initial network $f_1$ (e.g. by
running few SGD iterations) \\
\textbf{for} $t=1,\ldots,T$ \+ \\
define weights over the $m$ examples according to $D_i \propto
  e^{-y_i f_t(x_i)}$ \\
let $S$ be $n$ indices chosen at random according to the distribution
$D$ \\
use SGD for approximately finding a network as follows: \\
\begin{minipage}{0.9\textwidth}
\begin{equation} \label{eqn:obj} 
f_{t+1} \approx \argmin_{g}  \sum_{i \in S} y_i (f_t(x_i) - g(x_i)) +
\frac{1}{2} \sum_{i \in S} (g(x_i)-f_t(x_i))^2 
\end{equation} 
\end{minipage}
\\
\textbf{if} $\sum_{i=1}^m D_i \left[y_i(f_t(x_i)- f_{t+1}(x_i)) +
\frac{1}{2} (f_{t+1}(x_i)-f_t(x_i))^2\right]  < -\rho$ and $\forall i,
|f_{t+1}(x_i)-f_t(x_i)| \le 1$ \+ \\
continue to next iteration \- \\
\textbf{else} \+ \\
increase the number of SGD iterations and/or the architecture and try again
 \\
break if no such $f_{t+1}$ found 
\end{myalgo}

\section{Analysis}
For any classifier, define 
\[
M(f) = |\{i \in [m] : y_i f(x_i) \le 0\}| ~,
\]
to be the number of mistakes the classifier makes on the training
examples and 
\[
\textrm{err}(f) = \frac{M(f)}{m} 
\]
to be the error rate.

Our first theorem bounds $\textrm{err}(f)$ using the number of SelfieBoost
iterations and the edge parameter $\rho$.
\begin{theorem} \label{thm:main}
Suppose that SelfieBoost (for either the realizable or unrealizable
cases) is run for $T$ iterations with an edge parameter $\rho$. Then,
the number of mistakes of $f_{T+1}$ is bounded by 
\[
\textrm{err}(f_{T+1}) ~\le~ e^{-\rho\,T} ~.
\]
In other words, for any $\epsilon > 0$, if SelfieBoost performs 
\[
T \ge \frac{\log(1/\epsilon)}{\rho}
\]
succesful iterations then we must have $\textrm{err}(f_{T+1}) \le \epsilon$.
\end{theorem}
\begin{proof}
Define 
\[
L(f) = \log\left( \sum_i \exp(-y_i f(x_i)) \right) 
\]
and observe that $\log(M(f)) \le L(f)$ since for every $i$, $1[y_i
f(x_i) \le 0] \le \exp(-y_i f(x_i))$. 

Suppose that we start with $f_1$ such that $L(f_1) \le \log(m)$. For
example, we can take $f_1 \equiv 0$. At each succesful iteration of the
algorithm, we find $f_{t+1}$ such that $\forall i,
|f_{t+1}(x_i)-f_t(x_i)| \le 1$ and 
\begin{equation} \label{eqn:upperBoundinproof}
\sum_{i=1}^m D_i \left[y_i(f_t(x_i)- f_{t+1}(x_i)) +
\frac{1}{2} (f_{t+1}(x_i)-f_t(x_i))^2\right]  < -\rho ~.
\end{equation}
We will show that for such $f_{t+1}$ we have that $L(f_{t+1}) - L(f_t)
< -\rho$, which implies that after $T$ rounds we'll have $L(f_{T+1})
\le L(f_1) - \rho T \le \log(m) - \rho T$. But, we also have that
$L(f_{T+1}) \ge \log(M(f_{T+1}))$, hence
\[
\log( M(f_{T+1})) \le \log(m) - \rho T ~~~\Rightarrow~~~
\textrm{err}(f_{T+1})  = \frac{ M(f_{T+1})}{m} ~\le~ e^{-\rho T} ~,
\] 
as required. 

It is left to show that \eqref{eqn:upperBoundinproof} indeed upper
bounds $L(f_{t+1}) - L(f_t)$. To see this, we rely on the following
bound, which holds for every vectors $\theta,\lambda$ for which
$\theta_i - \lambda_i \le 1$ for all $i$ (see for example the proof of
Theorem 2.22 in \cite{shalev2011online}):
\[
\log(\sum_i e^{\lambda_i}) \le \log(\sum_i e^{\theta_i}) + \sum_i
\frac{e^{\theta_i}}{Z} (\lambda_i - \theta_i) + \frac{1}{2} \sum_i
\frac{e^{\theta_i}}{Z} (\lambda_i - \theta_i)^2 ~,
\]
where $Z = \sum_i e^{\theta_i}$. Therefore, with $\lambda_i =  -y_i
f_{t+1}(x_i)$, $\theta_i =  -y_i f_t(x_i)$, and $D_i = e^{-y_i f_t(x_i)}/Z$, we have
\begin{align} \nonumber
L(f_{t+1}) - L(f_t) &= \log(\sum_i e^{-y_i f_{t+1}(x_i)}) - \log(\sum_i
e^{-y_i f_t(x_i)})\\ \nonumber
&\le  \sum_{i} D_i \,y_i\, (f_t(x_i)-f_{t+1}(x_i)) + \frac{1}{2} \sum_{i} D_i   (f_t(x_i)-f_{t+1}(x_i))^2 ~. 
\end{align}
This concludes our proof. 
\end{proof}

Note that we obtain behavior which is very similar to AdaBoost, in
which the number of iterations depends on an ``edge'' --- here the
``edge" is how good we manage to minimize the objective at each
iteration, which corresponds to a ``regularized edge''. 

So far, we have shown that if SelfieBoost performs enough iterations
then it will produce an accurate classifier. Next, we show that it is
indeed possible to find a network with a small value of
\eqref{eqn:obj}. 
The key lemma below shows that SelfieBoost can progress. 
\begin{lemma}
Let $f$ be a network of size $k_1$ and $f^*$ be a network of size
$k_2$. Assume that $\textrm{err}(f^*)=0$. 
Then, there exists a network $g$ of size $k_1+k_2 + 1$ such that
$\forall i,
|\ell_i(f)-\ell_i(g)|\le 1$ and  \eqref{eqn:obj} is at most $-1/2$. 
\end{lemma}
\begin{proof}
Choose $g$ s.t. $g(x_i) = f(x_i) + y_i = f(x_i)+f^*(x_i)$. Clearly,
the size of $g$ is $k_1+k_2+1$. In addition, for every $i$,
\[
-y_i g(x_i)+y_i f(x_i) = -y_i^2 = -1 ~.
\]
This implies that \eqref{eqn:obj} becomes
\[
- \sum_i D_i + \frac{1}{2} \sum_i D_i = -\frac{1}{2} \sum_i D_i = -\frac{1}{2} ~.
\]
\end{proof}

The above lemma tells us that at each iteration of SelfieBoost, it is
possible to find a network for which the objective value of
\eqref{eqn:upperBoundinproof} is at most $-1/2$. Since SGD can rather
quickly find a network whose objective is bounded by a constant from
the optimum, we can use, say,  $\rho = 0.1$ and expect that SGD will find a
network with \eqref{eqn:upperBoundinproof} smaller than $-\rho$.

Observe that when we apply SGD we can either sample an example
according to $D$ at each iteration, or sample $n$ indices according to
$D$ and then let the SGD sample uniformly from these $n$ indices. 

\paragraph{Remark:} The above lemma shows us that we might need to
increase the network by the size of $f^*$ at each SelfieBoost
iterations. In practice, we usually learn networks which are
significantly larger than $f^*$ (because this makes the optimization
problem  SGD solves easier). Therefore, one may expect that even
without increasing the size of the network we'll be able to find a new
network with \eqref{eqn:upperBoundinproof} being negative.

\bibliographystyle{plainnat}
\bibliography{bib}

\begin{thebibliography}{27}
\providecommand{\natexlab}[1]{#1}
\providecommand{\url}[1]{\texttt{#1}}
\expandafter\ifx\csname urlstyle\endcsname\relax
  \providecommand{\doi}[1]{doi: #1}\else
  \providecommand{\doi}{doi: \begingroup \urlstyle{rm}\Url}\fi

\bibitem[Bengio(2009)]{Beng09}
Y.~Bengio.
\newblock Learning deep architectures for {AI}.
\newblock \emph{Foundations and Trends in Machine Learning}, 2\penalty0
  (1):\penalty0 1--127, 2009.

\bibitem[Bengio and LeCun(2007)]{bengio2007scaling}
Y.~Bengio and Y.~LeCun.
\newblock Scaling learning algorithms towards ai.
\newblock \emph{Large-Scale Kernel Machines}, 34, 2007.

\bibitem[Bengio et~al.(2013)Bengio, Courville, and
  Vincent]{bengio2013representation}
Y.~Bengio, A.~Courville, and P.~Vincent.
\newblock Representation learning: A review and new perspectives.
\newblock \emph{IEEE Transactions on Pattern Analysis and Machine
  Intelligence}, 35:\penalty0 1798--1828, 2013.

\bibitem[Bousquet and Bottou(2008)]{bousquet2008tradeoffs}
Olivier Bousquet and L{\'e}on Bottou.
\newblock The tradeoffs of large scale learning.
\newblock In \emph{Advances in neural information processing systems}, pages
  161--168, 2008.

\bibitem[Collobert and Weston(2008)]{ColWes08}
R.~Collobert and J.~Weston.
\newblock A unified architecture for natural language processing: deep neural
  networks with multitask learning.
\newblock In \emph{ICML}, 2008.

\bibitem[Dahl et~al.(2013)Dahl, Sainath, and Hinton]{dahl2013dropout}
G.~Dahl, T.~Sainath, and G.~Hinton.
\newblock Improving deep neural networks for lvcsr using rectified linear units
  and dropout.
\newblock In \emph{ICASSP}, 2013.

\bibitem[Duchi et~al.(2011)Duchi, Hazan, and Singer]{duchi2011adaptive}
John Duchi, Elad Hazan, and Yoram Singer.
\newblock Adaptive subgradient methods for online learning and stochastic
  optimization.
\newblock \emph{The Journal of Machine Learning Research}, 12:\penalty0
  2121--2159, 2011.

\bibitem[Freund and Schapire(1995)]{FreundSc95}
Y.~Freund and R.E. Schapire.
\newblock A decision-theoretic generalization of on-line learning and an
  application to boosting.
\newblock In \emph{European Conference on Computational Learning Theory
  (EuroCOLT)}, pages 23--37. Springer-Verlag, 1995.

\bibitem[Hinton et~al.(2006)Hinton, Osindero, and Teh]{HiOsTe06}
G.~E. Hinton, S.~Osindero, and Y.-W. Teh.
\newblock A fast learning algorithm for deep belief nets.
\newblock \emph{Neural Computation}, 18\penalty0 (7):\penalty0 1527--1554,
  2006.

\bibitem[Johnson and Zhang(2013)]{johnson2013accelerating}
Rie Johnson and Tong Zhang.
\newblock Accelerating stochastic gradient descent using predictive variance
  reduction.
\newblock In \emph{Advances in Neural Information Processing Systems}, pages
  315--323, 2013.

\bibitem[Krizhevsky et~al.(2012)Krizhevsky, Sutskever, and
  Hinton]{krizhevsky2012imagenet}
A.~Krizhevsky, I.~Sutskever, and G.~Hinton.
\newblock Imagenet classification with deep convolutional neural networks.
\newblock In \emph{NIPS}, 2012.

\bibitem[Le et~al.(2012)Le, Ranzato, Monga, Devin, Corrado, Chen, Dean, and
  Ng]{LRMDCCDN12}
Q.~V. Le, M.-A. Ranzato, R.~Monga, M.~Devin, G.~Corrado, K.~Chen, J.~Dean, and
  A.~Y. Ng.
\newblock Building high-level features using large scale unsupervised learning.
\newblock In \emph{ICML}, 2012.

\bibitem[LeCun and Bengio(1995)]{lecun1995convolutional}
Y.~LeCun and Y.~Bengio.
\newblock Convolutional networks for images, speech, and time series.
\newblock \emph{The handbook of brain theory and neural networks}, 3361, 1995.

\bibitem[LeCun et~al.(1998)LeCun, Bottou, Bengio, and
  Haffner]{lecun1998gradient}
Yann LeCun, L{\'e}on Bottou, Yoshua Bengio, and Patrick Haffner.
\newblock Gradient-based learning applied to document recognition.
\newblock \emph{Proceedings of the IEEE}, 86\penalty0 (11):\penalty0
  2278--2324, 1998.

\bibitem[LeCun et~al.(2012)LeCun, Bottou, Orr, and
  M{\"u}ller]{lecun2012efficient}
Yann~A LeCun, L{\'e}on Bottou, Genevieve~B Orr, and Klaus-Robert M{\"u}ller.
\newblock Efficient backprop.
\newblock In \emph{Neural networks: Tricks of the trade}, pages 9--48.
  Springer, 2012.

\bibitem[Lee et~al.(2009)Lee, Grosse, Ranganath, and Ng]{lee2009convolutional}
H.~Lee, R.~Grosse, R.~Ranganath, and A.Y. Ng.
\newblock Convolutional deep belief networks for scalable unsupervised learning
  of hierarchical representations.
\newblock In \emph{ICML}, 2009.

\bibitem[Martens(2010)]{martens2010deep}
James Martens.
\newblock Deep learning via hessian-free optimization.
\newblock In \emph{Proceedings of the 27th International Conference on Machine
  Learning (ICML-10)}, pages 735--742, 2010.

\bibitem[Ranzato et~al.(2007)Ranzato, Huang, Boureau, and
  Lecun]{ranzato2007unsupervised}
M.A. Ranzato, F.J. Huang, Y.L. Boureau, and Y.~Lecun.
\newblock Unsupervised learning of invariant feature hierarchies with
  applications to object recognition.
\newblock In \emph{CVPR}, 2007.

\bibitem[Schmidt et~al.(2013)Schmidt, Roux, and Bach]{schmidt2013minimizing}
Mark Schmidt, Nicolas~Le Roux, and Francis Bach.
\newblock Minimizing finite sums with the stochastic average gradient.
\newblock \emph{arXiv preprint arXiv:1309.2388}, 2013.

\bibitem[Schwenk and Bengio(2000)]{schwenk2000boosting}
Holger Schwenk and Yoshua Bengio.
\newblock Boosting neural networks.
\newblock \emph{Neural Computation}, 12\penalty0 (8):\penalty0 1869--1887,
  2000.

\bibitem[Shalev-Shwartz and Ben-David(2014)]{SSSSBD14}
S.~Shalev-Shwartz and S.~Ben-David.
\newblock \emph{Understanding Machine Learning: From Theory to Algorithms}.
\newblock Cambridge University Press, 2014.

\bibitem[Shalev-Shwartz(2011)]{shalev2011online}
Shai Shalev-Shwartz.
\newblock Online learning and online convex optimization.
\newblock \emph{Foundations and Trends in Machine Learning}, 4\penalty0
  (2):\penalty0 107--194, 2011.

\bibitem[Shalev-Shwartz and Wexler(2016)]{shalev2016minimizing}
Shai Shalev-Shwartz and Yonatan Wexler.
\newblock Minimizing the maximal loss: How and why.
\newblock In \emph{Proceedings of the 32nd International Conference on Machine
  Learning}, 2016.

\bibitem[Shalev-Shwartz and Zhang(2013)]{shalev2013stochastic}
Shai Shalev-Shwartz and Tong Zhang.
\newblock Stochastic dual coordinate ascent methods for regularized loss.
\newblock \emph{The Journal of Machine Learning Research}, 14\penalty0
  (1):\penalty0 567--599, 2013.

\bibitem[Sutskever et~al.(2013)Sutskever, Martens, Dahl, and
  Hinton]{sutskever2013importance}
I.~Sutskever, J.~Martens, G.~Dahl, and G.~Hinton.
\newblock On the importance of initialization and momentum in deep learning.
\newblock In \emph{ICML}, 2013.

\bibitem[Sutskever(2013)]{sutskever2013training}
Ilya Sutskever.
\newblock \emph{Training recurrent neural networks}.
\newblock PhD thesis, University of Toronto, 2013.

\bibitem[Zeiler and Fergus(2013)]{zeiler2013visualizing}
M.~Zeiler and R.~Fergus.
\newblock Visualizing and understanding convolutional neural networks.
\newblock \emph{arXiv preprint arXiv:1311.2901}, 2013.

\end{thebibliography}

\end{document}